\documentclass[twoside,11pt]{article}

\usepackage{bm}
\usepackage{amsmath}
\usepackage{jmlr2e}

\newcommand{\supp}{supp}

\newcommand\defeq{:=}

\DeclareMathOperator*{\argmin}{\mbox{arg\;min}}

\newcommand\vect[1]{\mathbf #1}

\newcommand{\vu}{\vect{u}}  
\newcommand{\vv}{\vect{v}}  

\newcommand{\vx}{\vect{x}}  
\newcommand{\vy}{\vect{y}}  
\newcommand{\vz}{\vect{z}}

\newcommand{\mD}{\mathbf{D}}

\newcommand{\mW}{\mathbf{W}}

\newcommand{\signalsize}{N}

\newcommand{\graphsigs}{\mathbb{R}^{\mathcal{V}}} 
\newcommand{\edgesigs}{\mathbb{R}^{\mathcal{E}}}
\newcommand{\edgesupport}{\mathcal{S}}
\newcommand{\dataset}{\mathcal{D}}
\newcommand{\edges}{\mathcal{E}}
\newcommand{\edgeset}{\mathcal{S}}
\newcommand{\nodes}{\mathcal{V}}
\newcommand{\cluster}{\mathcal{C}}
\newcommand{\graph}{\mathcal{G}}
\newcommand{\samplingset}{\mathcal{M}}
\newcommand{\boundary}{\partial \mathcal{F}}
\newcommand{\partition}{\mathcal{F}}
\newcommand{\xsig}{\vx}
\newcommand{\xsigval}[1]{x[{#1}]}

\newcommand{\mY}{\mathbf{Y}}

\newcommand{\mU}{\mathbf{U}}
\newcommand{\mQ}{\mathbf{Q}}

\usepackage{algorithm}
\usepackage{algpseudocode}
\floatname{algorithm}{Algorithm}
\algnewcommand\algorithmicinput{\textbf{Input:}}
\algnewcommand\INPUT{\item[\algorithmicinput]}
\algnewcommand\algorithmicoutput{\textbf{Output:}}
\algnewcommand\OUTPUT{\item[\algorithmicoutput]}

\ShortHeadings{Semi-supervised Learning via Sparse Label Propagation}{Jung, Hero, Mara and Jahromi}
\firstpageno{1}

\begin{document}

\title{Semi-Supervised Learning via Sparse Label Propagation}

\author{\name Alexander Jung \email first.last@aalto.fi \\
       \addr Department of Computer Science\\
       Aalto University\\
       Espoo, FI
        \AND
       \name Alfred O. Hero III \email last@eecs.umich.edu  \\
      \addr Department of EE and CS\\
       The University of Michigan\\
       Ann Arbor, MI
       \AND
       \name Alexandru Mara \email first.last@aalto.fi  \\
      \addr Department of Computer Science\\
       Aalto University\\
       Espoo, FI
       \AND
      \name Saeed Jahromi \email first.last@aalto.fi  \\
      \addr Department of Computer Science\\
       Aalto University\\
       Espoo, FI}

\editor{TBD}

\maketitle

\begin{abstract}
This work proposes a novel method for semi-supervised learning from partially labeled massive 
network-structured datasets, i.e., big data over networks. We model the underlying hypothesis, 
which relates data points to labels, as a graph signal, defined over some graph (network) structure intrinsic to the dataset. 
Following the key principle of supervised learning, i.e., ``similar inputs yield similar outputs'', we 
require the graph signals induced by labels to have small total variation. 
Accordingly, we formulate the problem of learning the labels of data points as a non-smooth convex optimization 
problem which amounts to balancing between the empirical loss, i.e., the discrepancy with some partially available 
label information, and the smoothness quantified by the total variation of the learned graph signal. 
We solve this optimization problem by appealing to a recently proposed 
preconditioned variant of the popular primal-dual method by Pock and Chambolle, which results in a 
sparse label propagation algorithm. 
This learning algorithm allows for a highly scalable implementation as message passing 
over the underlying data graph. By applying concepts of compressed sensing to the learning problem, 
we are also able to provide a transparent sufficient condition on the underlying 
network structure such that accurate learning of the labels is possible. We also 
present an implementation of the message passing formulation allows for a 
highly scalable implementation in big data frameworks. 

\end{abstract}


\section{Introduction}
\label{sec_intro}
 
A variety of modern technological systems generate (heterogeneous) data at unprecedented scale, i.e., 
``Big Data'' ~\citep{McKinseyBigdata,DonohoCursesBlessings,HadoopDefinitiveGuide,BigDataNetworksBook}. 
While lacking a precise formal definition, many Big Data problems share four main characteristics: 
(i) large data volume, (ii) high speed of data generation, (iii) heterogeneous data, i.e., 
partially labeled, mixture of audio, video and text data and (iv) noisy data, i.e., there are statistical 
variations due to missing labels, labelling errors, or 
poor data curation ~\citep{HadoopDefinitiveGuide}. Moreover, in a wide range of big 
data applications, e.g., social networks, sensor networks, communication networks, 
and biological networks, an intrinsic graph (or network) structure is present. 
This graph structure reflects either the physical properties of a system 
(e.g., public transportation networks) or statistical dependencies 
(e.g., probabilistic graphical models for bioinformatics). 
Quite often, these two notions of graph structure coincide: 
in a wireless sensor network, the graph modelling the communication links between nodes and 
the graph formed by statistical dependencies between sensor measurements 
resemble each other since both graphs are induced by the mutual proximity of the nodes ~\citep{WieselHero2012,Moldaschl2014,Quan20009}. 

On the algorithmic side, having a graph model for the observed data points 
facilitates scalable distributed data processing in the form of message 
passing on the graph. At a higher level, graph models are suitable for heterogeneous datasets 
containing data points of diverse nature, since they only require a weak notion 
of similarity between data points. Moreover, this notion of similarity enables the graph models to 
capitalize on massive amounts of unlabeled data via semi-supervised learning ~\citep{SemiSupervisedBook}. 
In particular, semi-supervised learning exploits the information contained in 
large amounts of unlabeled data points by considering 
their similarities to a small number of labeled data points. 

In this paper, we consider the problem of (transductive) semi-supervised learning using  
a known graph model for the raw data. This graph or network model, referred to as the data graph in what follows, 
contains individual data points as its nodes. Two nodes of the data graph which represent similar data points are 
connected by an edge.  The data graph typically contains a small number of data points for which label information is available. However, 
since acquiring label information is typically expensive, typically the majority of data points in the data graph are unlabeled. 
The goal of semi-supervised learning is then to learn the labels of those unlabeled data points by exploiting the 
label dependency information reflected by the topology of the data graph and the available label information for few data points. 
We tackle this learning problem by casting the dataset as a graph signal. In this graph signal model, 
the (unknown) labels of the data points are identified as the signal values of a 
graph signal whose domain is the node set of the data graph. The graph signal is 
then required to be nearly constant over 
well-connected subset of nodes (clusters).

Representing the labels of data points as a smooth graph signal is in line with the informal smoothness assumption underlying 
many supervised learning methods ~\citep{SemiSupervisedBook}, i.e., similar data points should have similar labels. 
In other words, the graph signal representing the underlying labeling is expected to 
reflect the topology of the underlying data graph, i.e., labels of closely connected nodes 
are highly correlated and thus these signals form close-knit clusters or communities ~\citep{Fortunato2009}. 
There are different options for quantifying the smoothness of a graph signal. A prominent line of work 
uses the graph Laplacian quadratic form of a graph signal as smoothness measure. 
In particular, the idea behind some well-known label propagation (LP) methods ~\citep{Zhu02learningfrom,SemiSupervisedBook} 
is to minimize this quadratic form while also enforcing consistency with the initial label information 
provided for a small number of data points. In contrast, the approach proposed in this paper is 
based on using the total variation ~\citep{shuman2013} of a graph signal, which models clustered 
graph signals more accurately compared to the Laplacian quadratic form ~\citep{SemiSupervisedBook}. 
Indeed, using TV as a smoothness measure allows the graph signal to change significantly on boundaries 
between clusters, as long as the number of boundary nodes is small compared to non-boundary nodes.  

\subsection{Existing Work} 

The problem of semi-supervised learning based on graph models is discussed in-depth in ~\citep{SemiSupervisedBook}, 
which also includes the popular LP method. Due to its simplicity and scalability, many variations of LP have been proposed. Underlying  
LP and its variants is the smoothness measured called the graph Laplacian quadratic form. 
In this proposed framework, LP can be interpreted as solving a system of linear equations obtained 
from the optimality condition for minimizing a quadratic function ~\citep{ChenMoura2014,ChenSanKov2015,ChenVarma}. 
In contrast, in this paper we will use total variation for quantifying the smoothness of a graph signal. 
This still results in solving a convex optimization problem, however with optimality conditions that lead to non-linear 
equations. Related previous approaches are the works ~\citep{SharpnackJMLR2012,TrendGraph,NetworkLasso} for general graph models, 
as well as a prominent line of work on total variation-based image processing ~\citep{chambolle2004algorithm,pock_chambolle_2016,PrecPockChambolle2011}. 
What sets our work apart from ~\citep{chambolle2004algorithm,pock_chambolle_2016,PrecPockChambolle2011}, 
which consider only grid graphs, is that we allow for general graph topology as in ~\citep{SharpnackJMLR2012,TrendGraph,NetworkLasso}. 
However, in contrast to ~\citep{SharpnackJMLR2012,TrendGraph}, the partial label information represented by graph signal values, 
is available only on a small subset of nodes. The proposed learning algorithm is closely related to the network Lasso ~\citep{NetworkLasso}, 
as will be discussed in Section \ref{sec_setup}.

\subsection{Contributions and Outline} 

In Section \ref{sec_setup}, we formulate (transductive) semi-supervised learning 
over graph-based models as a convex optimization problem. By applying a preconditioned 
variant of the popular primal-dual method of Pock and Chambolle to this problem, 
we obtain a sparse LP algorithm in Section \ref{sec_spl_Alg}. 
We then formulate this learning algorithm as message passing, requiring only local 
computations over the data graph and making it perfectly 
scalable to massive datasets (big data). In Section \ref{sec_main_results}, we present a 
sufficient condition on the available partial label information and data graph topology 
such that accurate learning is guaranteed by sparse LP. 
We also discuss how to implement sparse LP in a state-of-the-art big data framework. 
In order to verify the accuracy and scalability of the proposed learning algorithm, we conducted numerical 
experiments, with the obtained results reported in Section \ref{experimental_results}. 

\subsection{Notation}
Given a vector $\vx\!=\!(x_{1},\ldots,x_{n})^{T}\in \mathbb{R}^{n}$, we define the norms $\| \vx \|_{1} \defeq \sum_{l=1}^{n} |x_{l}|$ and 
$\| \vx \|_{2} \defeq \sqrt{ \sum_{l=1}^{n} (x_{l})^{2}}$, respectively. 
The spectral norm of a matrix $\mathbf{D}$ is denoted $\| \mD \|_{2} \defeq \sup_{\| \vx \|_{2}=1} \|\mD \vx \|_{2}$. 
For a positive semidefinite (psd) matrix $\mQ$, with spectral decomposition $\mQ = \mU{\rm diag} \{ q_{i} \}_{i=1}^{n} \mU^{T}$, 
we define its square root as $\mQ^{1/2} \defeq \mU {\rm diag} \{ \sqrt{q_{i}} \}_{i=1}^{n} \mU^{T}$. 
For a convex function $g(\vx)$, we denote its subdifferential at the point $\vx_{0} \in \mathbb{R}^{n}$ by 
\begin{equation} 
\partial g(\vx_{0}) \defeq \{ \vy \in \mathbb{R}^{n}: g(\vx) \geq g(\vx_{0})\!+\!\vy^{T}(\vx\!-\!\vx_{0}) \mbox{ for any } \vx \} \subseteq \mathbb{R}^{n}. \nonumber 
\end{equation}  

\section{The Semi-Supervised Learning Problem}
\label{sec_setup}

We consider a heterogeneous dataset $\dataset\!=\!\{ z_{i} \}_{i=1}^{\signalsize}\!\subseteq\!\mathcal{Z}$ 
consisting of $\signalsize$ data points $z_{i}\!\in\!\mathcal{Z}$, 
which might be of significantly different nature, e.g., the data point $z_{1} \in \mathbb{R}^{d}$ is a $d$-dimensional feature vector, data point $z_{2}$ is a continuous-time signal 
(i.e., $z_{2}: \mathbb{R}\!\rightarrow\!\mathbb{R}$) and another data point $z_{3}$ might represent the bag-of-words histogram of 
a text document. Thus, we assume the input space $\mathcal{Z}$ is rich enough to accommodate 
strongly heterogeneous data. Associated with the dataset $\mathcal{D}$ is an undirected weighted graph 
$\graph = (\nodes, \edges, \mathbf{W})$, referred to as the data graph, 
with node set $\nodes=\{1,\ldots,\signalsize\}$, edge set $\edges \subseteq \nodes \times \nodes$, 
and symmetric weight matrix $\mathbf{W}\in \mathbb{R}^{\signalsize \times \signalsize}$. 
The nodes represent the data points, i.e., each node $i \in \nodes$ corresponds to a unique datapoint $z_{i} \in \dataset$. 
An undirected edge $\{i,j\} \in \edges$ encodes some notion of 
(physical or statistical) proximity of neighboring data points $z_{i}$ and $z_{j}$. 
Given an edge $\{i,j\} \in \edges$, the nonzero value $W_{i,j}\!>\!0$ 
represents the strength of the connection $\{i,j\} \in \edges$. 
The edge set $\edges$ can be read off 
from the non-zero pattern of the weight matrix $\mathbf{W} \!\in\! \mathbb{R}^{\signalsize \times \signalsize}$ since 
\begin{equation}
\label{equ_edge_set_support_weights}
\{ i , j \} \in \edges \mbox{ if and only if } W_{i,j}  > 0.  
\end{equation} 
We assume the data graph $\graph$ to be simple, i.e., it contains no self-loops, so that $W_{i,i}\!=\!0$ for all $i\!\in\!\nodes$. 

The neighbourhood $\mathcal{N}(i)$ and weighted degree (strength) $d_{i}$ of node $i \in \nodes$ are defined, respectively, 
as ~\citep{NewmannBook}
\begin{equation} 
\label{equ_def_neighborhood}
\mathcal{N}(i) \defeq \{ j \in \nodes : \{i,j\} \!\in\!\edges \} \mbox{, and } d_{i} \defeq \sum_{j \in \mathcal{N}(i)} W_{i,j}. 
\end{equation} 
The (degree) distribution of the values $\{ d_{i} \}_{i \in \nodes}$ 
is of paramount relevance for the characterization and behaviour 
of (processes over) complex networks ~\citep{NewmannBook}. 
One key parameter of this distribution is the maximum strength of similarity over a neighborhood
\begin{equation}
\label{equ_def_max_node_degree}
 d_{\rm max} \defeq \max_{i \in \mathcal{V}} d_{i} \stackrel{\eqref{equ_def_neighborhood}}{=} \max_{i \in \nodes} \sum_{j \in \mathcal{N}(i)} W_{i,j} . 
\end{equation}
 
Starting with an undirected data graph $\graph=(\nodes,\edges,\mathbf{W})$, we define its directed version $\overrightarrow{\graph}=(\nodes,\overrightarrow{\edges},\mW)$, 
by introducing edge orientations, giving each each a direction with a head (origin node) and tail (destination node) denoted $e^{+}$ and  $e^{-}$, respectively. 
Given a set of edges $\edgeset \subseteq \edges$ in the undirected graph 
$\graph$, we denote the corresponding set of directed edges in $\overrightarrow{\graph}$ as $\overrightarrow{\edgeset}$. 
For a directed graph $\overrightarrow{\graph}$, we define the \emph{incidence matrix} 
$\mD \!\in\! \mathbb{R}^{|\edges| \times |\nodes|}$ ~\citep{SharpnackJMLR2012}
\begin{equation}
D_{e,i} = \begin{cases}  W_{e} & \mbox{ if } i = e^{+}  \\ 
				    - W_{e} & \mbox{ if } i = e^{-}  \\ 
				    0 &  \mbox{ else.}  \label{equ_def_incidence_mtx}
				    \end{cases}
\end{equation} 
\begin{figure}
\begin{center}
\includegraphics[width=.6\columnwidth,angle=0]{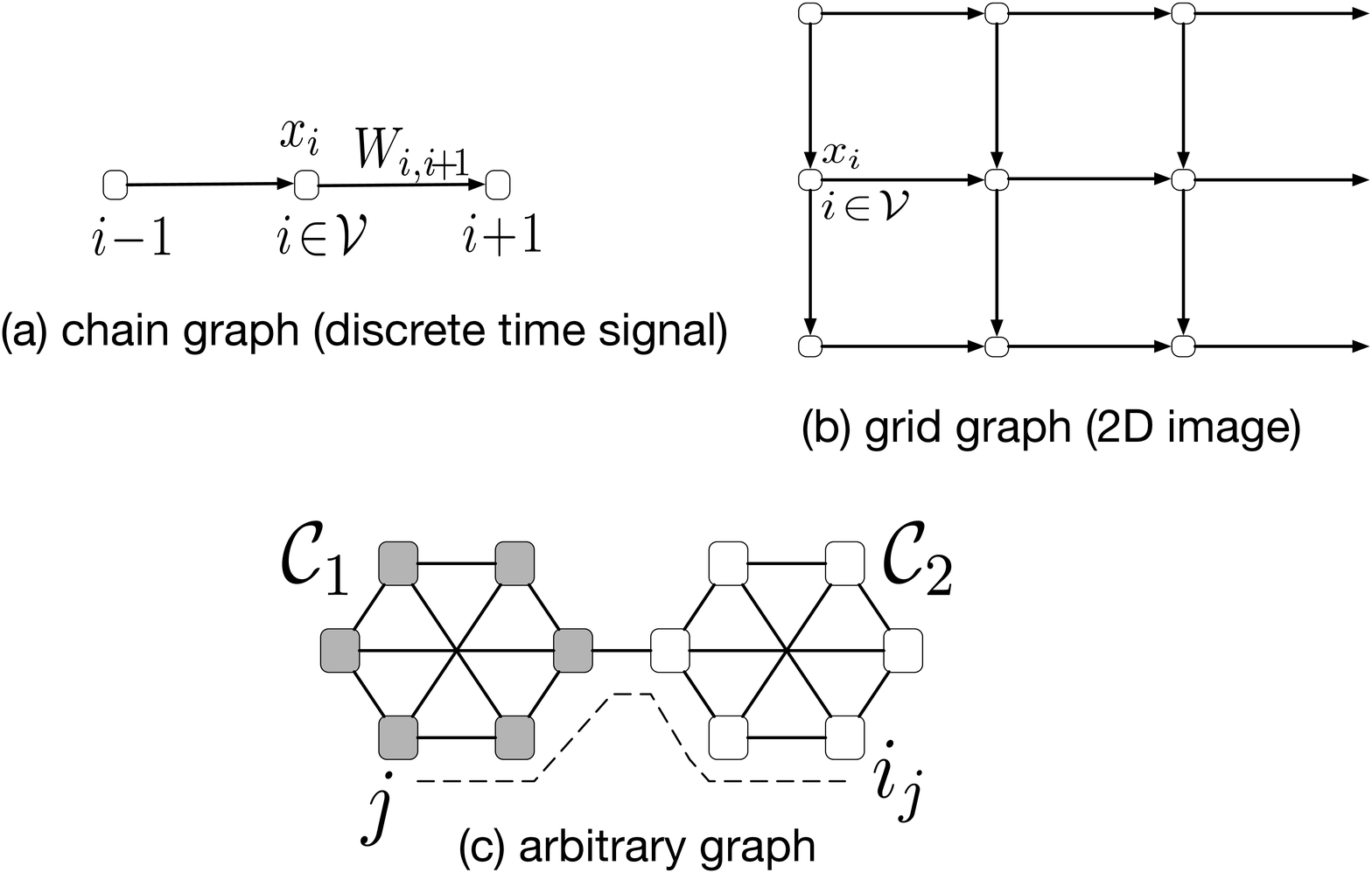}
\end{center}
\caption{\label{fig_graph_signals}Different types of datasets represented by 
graph signals defined over (a) a chain graph (discrete time signals), 
(b) grid graph (2D-images) and (c) a general graph (two cliques connected by a single edge).}
\end{figure}
Moreover, we define the oriented neighbourhoods of a node $i\in \nodes$ as $\mathcal{N}_{+}(i) \defeq \{ j \in \nodes: e=\{i,j\} \in \edges \mbox{, and } e^{+}=i \}$ and 
$\mathcal{N}_{-}(i) \defeq \{ j \in \nodes: e=\{i,j\} \in \edges \mbox{, and } e^{-}=i \}$.
If we number the nodes and orient the edges in the chain graph in Fig.\ \ref{fig_graph_signals}-(a) from left to right, 
its weighted incidence matrix would be 
\begin{equation}
\mathbf{D} = \begin{pmatrix} W_{1,2} & -W_{1,2} & 0  \\  0 & W_{2,3} & -W_{2,3} \end{pmatrix}. \nonumber
\end{equation} 

In many machine learning applications, we assign to each data point $z_{i} \in \dataset$ an associated label $x_{i} \in \mathbb{R}$, which we 
denote as an output value.\footnote{We 
highlight that the term ``label'' is typically reserved for discrete-valued or categorical output variables $x_{i}$ 
~\citep{BishopBook}. Since we can always represent the values of categorical output variables by real numbers, 
we will formulate our learning method for real-valued output variables $x_{i} \in \mathbb{R}$. 
In particular, the learning method summarized in Algorithm \ref{alg_sparse_label_propagation_centralized}, 
which is based on minimizing the total variation of real-valued labels, can also be used for 
classification by suitably quantizing the predicted output values. Extensions 
to other loss functions, more suitable to characterize the empirical 
error for discrete-valued or categorical labels, will be a focus of future work.} 
Consider, e.g., a dataset obtained from social network user profiles where each data point $z_{i}$ is constituted 
by all profile features. Here, the label $x_{i}$ might represent the preference of user $z_{i}$ to buy 
a certain product. 

We will represent the (mostly unknown) labels $x_{i}$ for the data points $z_{i} \in \dataset$ 
conveniently as a graph signal $\vx \!\in\! \mathbb{R}^{\nodes}$ defined over the data graph $\graph$. 
For a given data point $z_{i} \in \nodes$, represent by node $i\in \nodes$ in the data graph, 
the graph signal value $x[i]$ is given by the label $x_{i}$ associated with $z_{i}$. 
In what follows, we assume that for a small subset $\samplingset \subset \dataset$ of data points, which we refer to as the sampling set, 
we are provided initial labels $\{ x_{i} \}_{i \in \samplingset}$. With a slight abuse of notation, 
we will refer by $\samplingset$ also to the subset of nodes $i\in \nodes$ representing labelled data points $z_{i} \in \samplingset$.  
In applications where acquiring labels is costly, we would like to keep the sampling set 
as small as possible, i.e., we are mainly interested in the regime where $|\samplingset| \ll \signalsize$.

In order to learn the labels $\vx \in \graphsigs$ from a limited number of 
initial labels $\{ x_{i} \}_{i \in \samplingset}$, we rely on an informal smoothness assumption 
behind many (semi-) supervised learning methods ~\citep{SemiSupervisedBook}: 

\noindent
{\bf Informal Smoothness Assumption.}
{\it Consider a graph signal $\vx \in \mathbb{R}^{\nodes}$ whose signal values are the (unknown) labels $x_{i}$ of the 
data points $z_{i} \in \dataset$. The signal values $x[i]$, $x[j]$ at nodes $i,j \in \nodes$ within a well-connected subset of 
nodes in the data graph are similar, i.e., $x[i] \approx x[j]$ with respect to some metric on the label space $\mathbb{R}$.}

In order to make this informal smoothness assumption precise, 
we propose to measure the smoothness of a graph signal 
$\vx \in \graphsigs$ via its total variation (TV) 
\begin{equation} 
\label{equ_def_TV}
\| \vx \|_{\rm TV} \defeq \sum_{\{i,j\} \in \edges} W_{i,j}  | x[j]\!-\!x[i]|.
\end{equation} 
The use of (variants of) TV is prevalent in network clustering and image processing ~\citep{SharpnackJMLR2012,TrendGraph,NetworkLasso}. 
A generalization of the TV for vector-valued graph signals is minimized in ~\citep{NetworkLasso} in order to 
learn improved classifiers (e.g., support vector machines). For the particular grid graph structure 
underlying image data, many efficient denoising techniques are based on minimizing TV ~\citep{pock_chambolle_2016,pock_chambolle}. 

Some well-known examples of smooth graph signals are low-pass 
signals in digital signal processing where time samples at 
adjacent time instants are strongly correlated for sufficiently 
high sampling rate (cf.\ Fig.\ \ref{fig_graph_signals}-(a)). 
Another example are 2D signals (modelled as grid graphs) representing natural 
images where close-by pixels tend to be coloured similarly 
(cf.\ Fig.\ \ref{fig_graph_signals}-(b)).

Les us now introduce a very simple but useful model for smooth graph signals $\vx \in \graphsigs$ having small TV $\| \vx \|_{\rm TV}$. 
In particular, consider graph signals of the form   
\begin{equation}
\label{equ_def_clustered_signal_model}
\xsigval{i} = \sum_{l=1}^{|\partition|} a_{l} \mathcal{I}_{\cluster_{l}}[i] \quad\quad \mbox{with } \mathcal{I}_{\cluster_{l}}[i] \defeq \begin{cases} 1 \mbox{ for } i \in \cluster_{l} \\ 0 \mbox{ else.}  \end{cases}
\end{equation} 
Here, we used a fixed partition $\partition = \{ \cluster_{1},\ldots,\cluster_{|\partition|} \}$ of the entire data 
graph $\graph$ into disjoint clusters $\cluster_{l} \subseteq \nodes$. 
The boundary $\boundary$ of a partition $\partition$ is constituted by all edges connecting 
nodes from different clusters, i.e., 
\begin{equation}
\label{equ_def_boundary}
\boundary \defeq \{ \{i,j\} \in \edges \mbox{ with  } i \in \cluster_{a} \mbox{ and } j \in \cluster_{b} \neq \cluster_{a} \}.
\end{equation} 
We can bound the TV of a clustered graph signal of the form \eqref{equ_def_clustered_signal_model}, 
as can be verified easily, by 
\begin{equation}
\label{equ_bound_TV_norm_clustered}
\| \xsig \|_{\rm TV} = 2 \max_{l\in\{1,\ldots,|\partition|\}} |a_{l}|  \sum_{\{i,j\} \in \boundary} W_{i,j}. 
\end{equation} 
While there is now unique formal definition of when to call a subset of nodes a ``cluster'', we typically 
refer by a cluster to a subset of nodes which is weakly connected to nodes outside the cluster ~\citep{Fortunato2009}. 
For weakly interconnected clusters $\cluster_{l}$, the partition $\partition = \{ \cluster_{1},\ldots,\cluster_{|\partition|} \}$ 
has a small boundary $\boundary$ such that graph signals clustered according to $\partition$ 
(cf.\ \eqref{equ_def_clustered_signal_model}) have small TV due to \eqref{equ_bound_TV_norm_clustered}. 

The signal model \eqref{equ_def_clustered_signal_model}, which also has been 
studied in ~\citep{SharpnackJMLR2012,TrendGraph}, is closely related to the stochastic 
block model (SBM) ~\citep{Mossel2012}. Indeed, the SBM corresponds to graph signals of the form  
\eqref{equ_def_clustered_signal_model} with coefficients $a_{l}$ being unique 
for each cluster $\cluster_{l}$, i.e., $a_{l} \neq a_{l'}$ for $\cluster_{l} \neq \cluster_{l'}$. 
Moreover, the SBM provides a generative (stochastic) model 
for the edges within and between  the clusters $\cluster_{l}$. 

The incidence matrix $\mD$ (cf.\ \eqref{equ_def_incidence_mtx}) of the oriented data 
graph $\overrightarrow{\graph}$ allows us to represent $\|\vx\|_{\rm TV}$ by   
\begin{equation}
\label{equ_repr_ell_1_TV}
\| \vx \|_{\rm TV}  = \| \mathbf{D} \vx \|_{1}. 
\end{equation} 
We learn the labels $\hat{x}[i]$ for the entire dataset 
$\mathcal{D}=\{z_{i}\}_{i=1}^{N}$ via minimizing $\| \tilde{\vx} \|_{\rm TV}$ among 
all labellings $\tilde{\vx} \in \graphsigs$ which are consistent with the initial labels $\{ x_{i} \}_{i \in \samplingset}$, i.e., 
\begin{align} 
\hat{\vx} & \in \argmin_{\tilde{\vx} \in \graphsigs} \sum_{\{i,j\} \in \edges} W_{i,j}  | \tilde{x}[j]\!-\!\tilde{x}[i]|
\quad \mbox{s.t.} \quad \tilde{x}[i]\!=\!x[i]  \mbox{ for all } i \in \samplingset \nonumber \\
& \stackrel{\eqref{equ_repr_ell_1_TV}}{=} 
\argmin_{\tilde{\vx} \in \graphsigs} \| \mathbf{D} \tilde{\vx} \|_{1} \quad \mbox{s.t.} \quad  \tilde{x}[i]
\!=\!x[i] \mbox{ for all } i \in \samplingset.  \label{equ_min_constr}
\end{align}
The objective function of the optimization problem \eqref{equ_min_constr} is the seminorm $\|\vx\|_{\rm TV}$, 
which is a convex function.\footnote{The seminorm $\|\vx\|_{\rm TV}$ is convex since 
it is homogeneous ($\| \alpha\vx\|_{\rm TV}\!=\!|\alpha|\|\vx\|_{\rm TV}$ for $\alpha \in \mathbb{R}$) 
and satisfies the triangle inequality ($\|\vx\!+\!\vy\|_{\rm TV} \!\leq\! \|\vx\|_{\rm TV}\!+\!\|\vy\|_{\rm TV}$).
These two properties imply convexity ~\citep[Section 3.1.5]{BoydConvexBook}.}  
Since moreover the constraints in \eqref{equ_min_constr} are linear, the optimization problem \eqref{equ_min_constr} 
is a convex optimization problem ~\citep{BoydConvexBook}. Moreover, as the notation in \eqref{equ_min_constr} suggests, 
there typically exist several solutions to this optimization problem. The optimization method we will use for solving \eqref{equ_min_constr} 
does not require uniqueness of the solution, i.e., it can be applied even if there are several 
different optimal labellings $\hat{\vx}$. Any labelling $\hat{\vx}$ which solves \eqref{equ_min_constr} is 
characterized by two properties: (i) it is consistent with the initial labels, i.e., $\hat{x}[i]=x_{i}$ for all $i\in \samplingset$; and (ii) 
it has minimum TV among all such labelings. 

Learning methods based on solving \eqref{equ_min_constr} will be called sparse LP (SLP) methods, due to 
the similarity of \eqref{equ_min_constr} and the optimization problem underlying ordinary LP ~\citep[Chap 11.3.4.]{SemiSupervisedBook}: 
\begin{equation}
\hat{\vx}_{\rm LP}  \in \argmin_{\tilde{\vx} \in \graphsigs} \sum_{\{i,j\} \in \edges} W_{i,j} (\tilde{x}[i]-\tilde{x}[j])^2  \quad \mbox{s.t.} \quad  \tilde{x}[i]
\!=\!x_{i} \mbox{ for all } i \in \samplingset.  \label{equ_LP_problem}
\end{equation} 
We highlight the fact that problem \eqref{equ_LP_problem} amounts to minimizing the squared $\ell_{2}$ norm, while the 
SLP problem \eqref{equ_min_constr} aims to minimize the $\ell_{1}$ norm of the weighted graph signal differences $x[i]\!-\!x[j]$ 
over all edges $\{i,j\} \in \edges$ in the data graph $\graph$. The use of different norms results in significantly different behaviour of 
LP (aiming at solving \eqref{equ_LP_problem}) compared to the proposed SLP (aiming at solving \eqref{equ_min_constr}).
In particular, since SLP is based on the $\ell_{1}$ norm of signal differences, it is able to accurately learn graph signals 
which might vary abruptly over few edges, e.g., piecewise constant discrete time signals. In contrast, LP 
tends to smooth out such abrupt signal variations (cf.\ Section \ref{experimental_results}). 

The SLP problem \eqref{equ_min_constr} is closely related to the recently proposed network Lasso ~\citep{NetworkLasso,WhenIsNLASSO}
\begin{equation}
\hat{\vx}_{\rm nLasso} \in  \argmin_{\tilde{\vx} \in \graphsigs} \sum_{i \in \samplingset}  (\tilde{x}[i]-x_{i})^2  + \lambda \| \tilde{\vx}\|_{\rm TV} \label{equ_nLasso}.
\end{equation} 
The tuning parameter $\lambda$ trades small empirical label fitting error $\sum_{i \in \samplingset}  (\tilde{x}[i]\!-\!x_{i})^2$ against 
smoothness $\|\hat{\vx}_{\rm nLasso}\|_{\rm TV}$ of the learned graph signal $\hat{\vx}_{\rm nLasso}$. 
Choosing a large value of $\lambda$ enforces smoothness of the learned graph signal, while using a small value for $\lambda$ 
puts more emphasis on the empirical error. In particular, as $\lambda \rightarrow 0$ the solutions of the network 
Lasso problem \eqref{equ_nLasso} become more similar to the solutions of the SLP problem \eqref{equ_min_constr}. 

Now that we have formulated the (transductive) semi-supervised learning problem as the 
(nonsmooth) convex optimization problem \eqref{equ_min_constr}, we can apply efficient convex optimization 
methods to obtain learning algorithms. However, it is important to note that the 
formulation \eqref{equ_min_constr} requires knowledge of the data graph $\graph$. In particular, 
the objective function $\| \vx \|_{\rm TV}= \|\mathbf{D} \vx \|_{1}$ depends on the 
graph structure via the incidence matrix $\mathbf{D}$ (cf.\ \eqref{equ_def_incidence_mtx}). 
Many applications naturally specify the data graph structure, e.g., proximal dependency in chain graphs 
of autoregressive time series models or grid graphs of Markov random fields for image processing. 
However, in some applications the underlying 
graph model has to be learned from the data, e.g., using probabilistic 
graphical models ~\citep{CSGraphSelJournal,JungGaphLassoSPL,JuHeck2014}. 
In this paper, we assume the data graph $\graph$ is specified. 

For a dataset $\mathcal{D}$ whose data graph $\graph$ contains several isolated components ~\citep{NewmannBook} 
that are not connected with each other, the learning problem \eqref{equ_min_constr} decomposes into independent 
subproblems, i.e., one instance of \eqref{equ_min_constr} 
for each of the components. Therefore, without loss of generality, in the sequel we consider 
datasets whose data graph $\graph$ is connected. 

\section{Sparse Label Propagation for Semi-Supervised Learning}
\label{sec_spl_Alg}

The learning problem \eqref{equ_min_constr} is a non-smooth convex optimization problem, which 
precludes the use of standard gradient methods such as (accelerated) gradient descent. 
However, the objective function and the constraint set have simple structure 
when considered independently from each other. This structure suggests the use of efficient 
proximal methods ~\citep{ProximalMethods} for solving \eqref{equ_min_constr}. In particular, 
we apply a preconditioned variant of the primal-dual method introduced by ~\citep{pock_chambolle}
to solve \eqref{equ_min_constr}.  

In order to apply the primal-dual method of ~\citep{pock_chambolle}, we have to reformulate \eqref{equ_min_constr}
as an equivalent unconstrained convex optimization problem 
\begin{align}
\label{equ_min_constr_unconstr}
\hat{\vx} & \!\in\! \argmin_{\vx \in \mathbb{R}^{\nodes}} f(\vx) \defeq g(\mD \vx) + h(\vx) \mbox{, with } g(\vy) \defeq \| \vy\|_{1} \mbox{ and } h(\vx) \defeq \delta_{\mathcal{Q}}(\vx).  
\end{align}
Here, the indicator function $\delta_{\mathcal{Q}}(\vx)$ of a convex set $\mathcal{Q}$ is defined as  
\begin{equation}
\delta_{\mathcal{Q}}(\vx) = \begin{cases} \infty  \mbox{ if } \vx \notin \mathcal{Q} \\ 0 \mbox{ if } \vx \in \mathcal{Q} \end{cases}, 
\end{equation}
or, equivalently, via its epigraph 
\begin{equation}
{\rm epi} \delta_{\mathcal{Q}} \defeq \{ (\vx,t):  \vx \in \mathcal{Q}, t \in \mathbb{R}_{+} \} \subseteq \graphsigs \times \mathbb{R}. \nonumber
\end{equation}
The solutions $\hat{\vx}$ of \eqref{equ_min_constr_unconstr} are characterized by the zero-subgradient condition 
\begin{equation}
\label{equ_zero_subgradient}
\mathbf{0} \in \partial f(\hat{\vx}). 
\end{equation} 
A proximal method is obtained by applying fixed-point iterations for some operator $\mathcal{P}$ 
such that $\hat{\vx}$ are the fixed-points of $\mathcal{P}$, i.e.,  
\begin{equation}
\label{equ_zero_subgradient_equ}
\mathbf{0} \in \partial f(\hat{\vx}) \mbox{ if and only if } \hat{\vx} = \mathcal{P} \hat{\vx}. 
\end{equation} 
The operator $\mathcal{P}$ depends on the objective function $f(\vx)$ but is not uniquely defined, i.e., there are 
different choices for $\mathcal{P}$ such that \eqref{equ_zero_subgradient_equ} will define solutions of \eqref{equ_min_constr_unconstr}. 

One extremely powerful approach to choosing the operator $\mathcal{P}$ in \eqref{equ_zero_subgradient_equ}, which underlies many 
primal-dual optimization methods \citep{pock_chambolle,PrecPockChambolle2011}, is based on convex duality ~\citep{RockafellarBook}. 
In particular, according to ~\citep[Thm.\ 31.3]{RockafellarBook}, a graph signal $\hat{\vx} \in \graphsigs$ is solution to the problem \eqref{equ_min_constr_unconstr} 
if and only if there exists a (dual vector) $\hat{\vy} \in \mathbb{R}^{\edges}$ such that 
\begin{equation}
\label{equ_two_coupled_conditions}
\mD \hat{\vx} \in \partial g^{*}(\hat{\vy}) \mbox{ , and } -(\mD^{T} \hat{\vy}) \in \partial h(\hat{\vx}). 
\end{equation} 
Here, $g^{*}(\hat{\vy})$ denotes the convex conjugate of the convex function $g(\vy)$, which is defined as ~\citep{BoydConvexBook}
\begin{equation}
g^{*}(\hat{\vy}) \defeq \sup_{\vy \in \mathbb{R}^{\edges}} \vy^{T}\hat{\vy}- g( \vy).  
\end{equation} 
The dual vector $\hat{\vy} \in \mathbb{R}^{\edges}$ represents a signal defined over the edges $\edges$ in 
the data graph $\graph$, with the entry $\hat{y}[e]$ being the signal value associated with the particular edge $e \in \edges$. 
The two coupled conditions in \eqref{equ_two_coupled_conditions} can be rewritten, rather trivially, as 
\begin{equation}
\label{equ_manipulated_coupled_conditions}
2 {\bm \Lambda}  \mD \hat{\vx} +\hat{\vy} \in {\bm \Lambda} \partial g^{*}(\hat{\vy})+ {\bm \Lambda} \mD\hat{\vx}+\hat{\vy} 
\mbox{ , and } \hat{\vx} - {\bm \Gamma} \mD^{T} \hat{\mY} \in \hat{\vx} + {\bm \Gamma} \partial h(\hat{\vx}), 
\end{equation}
with the invertible diagonal matrices 
\begin{equation} 
\label{equ_def_scaling_matrices}
{\bf \Lambda} \defeq {\rm diag} \{ \lambda_{\{i,j\}} \defeq1/(2W_{i,j}) \}_{\{i,j\} \in \edges} \in \mathbb{R}^{\edges \times \edges} \mbox{ and } {\bf \Gamma} \defeq {\rm diag} \{ \gamma_{i} \defeq 1/d_{i} \}_{i \in \nodes} \in \mathbb{R}^{\nodes \times \nodes}.
\end{equation}

Since the resolvent operators $(\mathbf{I} + {\bm \Gamma} \partial h)^{-1}$ and $(\mathbf{I}+ {\bm \Lambda}\partial g^{*})^{-1}$ 
are well-defined for the particular 
choices of $g(\vy)$ and $h(\vx)$ in \eqref{equ_min_constr_unconstr} (cf.\ ~\citep{ProximalMethods}), 
the condition \eqref{equ_manipulated_coupled_conditions} for $\hat{\vx}$, $\hat{\vy}$ to be primal and dual 
optimal is equivalent to  
\begin{align}
\label{equ_condition_fix_point}
\hat{\vx} &= (\mathbf{I} + {\bm \Gamma} \partial h)^{-1} (\hat{\vx} - {\bm \Gamma} \mD^{T} \hat{\vy}) \\ 
\hat{\vy} - 2(\mathbf{I} + {\bf \Lambda}  \partial g^{*})^{-1}   {\bf \Lambda}  \mD \hat{\vx} & = (\mathbf{I} +{\bf \Lambda}  \partial g^{*})^{-1}(\hat{\vy} -  {\bf \Lambda}\mD\hat{\vx}) \nonumber
\end{align}  
The characterization \eqref{equ_condition_fix_point} for $\hat{\vx} \in \graphsigs$ to be a solution 
to \eqref{equ_min_constr} lends naturally to the following fixed-point iterations for finding $\hat{\vx}$ (cf.\ ~\citep{PrecPockChambolle2011}) 
\begin{align}
\hat{\vx}^{(k+1)} &\defeq (\mathbf{I} + {\bm \Gamma} \partial h)^{-1} (\hat{\vx}^{(k)} - {\bm \Gamma} \mD^{T} \hat{\vy}^{(k)})  \nonumber \\ 
\hat{\vy}^{(k+1)} &\defeq (\mathbf{I} + {\bf \Lambda}  \partial g^{*})^{-1}(\hat{\vy}^{(k)} +  {\bf \Lambda}  \mD(2\hat{\vx}^{(k+1)}- \hat{\vx}^{(k)})). \label{equ_fixed_point_iterations}
\end{align}  
The specific choice \eqref{equ_def_scaling_matrices} for the scaling matrices ${\bf \Gamma}$ and ${\bf \Lambda}$ in 
\eqref{equ_fixed_point_iterations} can be shown to satisfy 
\begin{equation}
\label{equ_convergence_condition}
\| {\bf \Gamma}^{1/2} \mD^{T} {\bf \Lambda}^{1/2} \|_{2} < 1,
\end{equation}
which guarantees, in turn, convergence of the iterations \eqref{equ_fixed_point_iterations}  (cf.\ ~\citep{PrecPockChambolle2011}). 

The fixed-point iterations \eqref{equ_fixed_point_iterations} are similar to those considered in ~\citep[Sec. 6.2.]{pock_chambolle}. 
In particular, the iterations \eqref{equ_fixed_point_iterations} involve the resolvent operators ~\citep[Sec. 1.1.]{PrecPockChambolle2011}
\begin{align}
(\mathbf{I} + {\bm \Gamma} \partial h)^{-1} (\vx) & = \argmin\limits_{\vz \in \graphsigs} h(\vz) + (1/2) (\vx\!-\!\vz)^{T} {\bm \Gamma}^{-1}(\vx\!-\!\vz) \mbox{, and }  \nonumber \\ 
(\mathbf{I} + {\bf \Lambda} \partial g)^{-1} (\vy) & = \argmin\limits_{\vz \in \edgesigs} g(\vz) +  (1/2) (\vy\!-\!\vz)^{T} {\bm \Lambda}^{-1}(\vy\!-\!\vz). \label{equ_iterations_number_112}
\end{align}
Based on \eqref{equ_iterations_number_112}, closely following the arguments used in ~\citep[Sec. 6.2.]{pock_chambolle}, we obtain 
simple closed-form expressions for the updates in \eqref{equ_fixed_point_iterations} yielding, in turn, 
Algorithm \ref{alg_sparse_label_propagation_centralized}. 

There are various options for the stopping criterion in Algorithm \ref{alg_sparse_label_propagation_centralized}, 
e.g., using a fixed number of iterations or testing for sufficient decrease of the objective 
function (cf.\ ~\citep{becker2011nesta}). When using a fixed number of iterations, it is helpful 
to have a characterization of the convergence rate of Algorithm \ref{alg_sparse_label_propagation_centralized}. 
Such a characterization can be obtained from ~\citep[Lemma 2]{PrecPockChambolle2011}, which yields 
the following upper bound on the sub-optimality in terms of the TV $\| \hat{\vx} \|_{\rm TV} $ (cf.\ \eqref{equ_min_constr}) 
achieved by the output $\hat{\vx}^{(k)}$ after $k$ iterations of Algorithm \ref{alg_sparse_label_propagation_centralized} : 
\begin{equation} 
\label{equ_conv_rate_inv_prop_k}
\| \hat{\vx}^{(k)} \|_{\rm TV} - \| \hat{\vx} \|_{\rm TV} \leq c_{1}/k.
\end{equation} 
The constant $c_{1}$ may depend on the underlying 
graph signal $\vx$. We highlight that this convergence rate, i.e., $\propto 1/k$ with $k$ being the number of iterations,  
is optimal for first-order gradient methods, i.e., which only have access  to the (sub-)gradient of the 
objective function in each iteration, applied to problems of the type \eqref{equ_min_constr} (cf.\ ~\citep{NemYudFOM}). 
\begin{algorithm}[h]
\caption{Sparse Label Propagation}{}
\begin{algorithmic}[1]
\renewcommand{\algorithmicrequire}{\textbf{Input:}}
\renewcommand{\algorithmicensure}{\textbf{Output:}}
\Require  oriented data graph $\overrightarrow{\graph}$ with incidence matrix $\mD\!\in\! \mathbb{R}^{\overrightarrow{\edges} \times \nodes}$ 
(cf.\ \eqref{equ_def_incidence_mtx}), sampling set $\samplingset$, initial labels $\{ x_{i} \}_{i \in \samplingset}$. 
\Statex\hspace{-6mm}{\bf Initialize:} $k\!\defeq\!0$, 
$\vz^{(0)}\!\defeq\!\mathbf{0}$, $\vx^{(0)}\!\defeq\!\vx_{\mathcal{M}}$, $\hat{\vx}^{(0)}\!\defeq\!0$, $\hat{\vy}^{(0)}=\mathbf{0}$, 
$\gamma_{i} \defeq 1/\sum_{j \in \mathcal{N}(i)} W_{i,j}$, $\lambda_{\{i,j\}} = 1/(2W_{i,j})$. 
\Repeat
\vspace*{2mm}
\State $\hat{\vx}^{(k+1)}  \defeq \hat{\vx}^{(k)} - {\bm \Gamma} \mD^{T} \hat{\vy}^{(k)}$ with ${\bm \Gamma}={\rm diag} \{ \gamma_{i} \}_{i \in \nodes}$
\vspace*{2mm}
\State $\hat{x}^{(k+1)}[i] \defeq x_{i}$  for all sampled nodes $i \in \samplingset$
\vspace*{2mm}
\State  $\tilde{\vx}  \defeq 2 \hat{\vx}^{(k+1)} - \hat{\vx}^{(k)}$
\vspace*{2mm}
\State $\hat{\vy}^{(k+1)}  \defeq \hat{\vy}^{(k)} + {\bf \Lambda}  \mD  \tilde{\vx}$ with ${\bf \Lambda}={\rm diag} \{ \lambda_{\{i,j\}} \}_{\{i,j\} \in \edges}$
\vspace*{2mm}
\State $y^{(k+1)}[e]  \defeq y[e] / \max\{1, |y[e]| \}$  for all edges  $e \in \overrightarrow{\edges}$
\vspace*{2mm}
\State $k \defeq k+1$ 
\vspace*{2mm}
\Until{stopping criterion is satisfied}
\vspace*{2mm}
\Ensure labels $\hat{x}_{\rm SLP}[i] \defeq \hat{x}^{(k)}[i]$ \mbox{ for all }$i \in \nodes$
\end{algorithmic}
\label{alg_sparse_label_propagation_centralized}
\end{algorithm}
Let us highlight that the SLP Algorithm \ref{alg_sparse_label_propagation_centralized} only requires the data graph but it 
does not require knowledge of the partition $\partition$ underlying the clustered signal model \eqref{equ_def_clustered_signal_model}. 
Rather, the partition $\partition$ and corresponding signal model \eqref{equ_def_clustered_signal_model} is only used for the 
performance analysis of Algorithm \ref{alg_sparse_label_propagation_centralized}. 

Below we present an implementation of Algorithm \ref{alg_sparse_label_propagation_centralized} 
as message passing over the underlying data graph $\graph$. This message passing implementation, 
summarized in Algorithm \ref{sparse_label_propagation_mp}, is obtained by implementing the application of the graph 
incidence matrix $\mD$ and its transpose $\mD^{T}$ (cf. steps $2$ and $5$ of Algorithm \ref{alg_sparse_label_propagation_centralized}) 
by local updates of the labels $\hat{x}[i]$, i.e., updates which involve only the neighbourhoods 
$\mathcal{N}(i)$, $\mathcal{N}(j)$ of all edges $\{i,j\} \in \edges$ in the data graph $\graph$.  
In particular, the execution of Algorithm \ref{sparse_label_propagation_mp} does not require to collect 
global knowledge about the entire data graph (e.g., the maximum strength $d_{\rm max}$ \eqref{equ_def_max_node_degree}) at 
some central processing unit. Indeed, if we associate each node in the data graph with a computational unit, 
the execution of Algorithm \ref{sparse_label_propagation_mp} requires each node $i \in \nodes$ only to store the 
values $\{ \hat{y}[\{i,j\}], W_{i,j} \}_{j \in \mathcal{N}(i)}$ and $\hat{x}^{(k)}[i]$.  
This characteristic allows Algorithm \ref{sparse_label_propagation_mp} to scale well for massive datasets (big data).

\begin{algorithm}[h]
\caption{Sparse Label Propagation as Message Passing}{}
\begin{algorithmic}[1]
\renewcommand{\algorithmicrequire}{\textbf{Input:}}
\renewcommand{\algorithmicensure}{\textbf{Output:}}
\Require oriented data graph $\overrightarrow{\graph}=(\nodes,\overrightarrow{\edges},\mW)$, sampling set $\samplingset$, initial labels $\{ x_{i} \}_{i \in \samplingset}$. 
\Statex\hspace{-6mm}{\bf Initialize:} $k\!\defeq\!0$, 
$\vz^{(0)}\!\defeq\!\mathbf{0}$, $\vx^{(0)}\!\defeq\!\vx_{\mathcal{M}}$, $\hat{\vx}^{(0)}\!\defeq\!0$, $\hat{\vy}^{(0)}=\mathbf{0}$, 
$\gamma_{i} \defeq 1/\sum_{j \in \mathcal{N}(i)} W_{i,j}$, $\lambda_{\{i,j\}} = 1/(2W_{i,j})$. 
\Repeat
\vspace*{2mm}
\State for all nodes $i\!\in\!\nodes$: $\tilde{x}^{(k+1)}[i]\!\defeq\!\tilde{x}^{(k)}[i]\!-\!\gamma_{i} \bigg[ \sum\limits_{j \in \mathcal{N}_{+}(i)} \hspace*{-3mm}W_{i,j} \hat{y}^{(k)}[\{i,j\}]- \hspace*{-3mm}\sum\limits_{j \in \mathcal{N}_{-}(i)} \hspace*{-3mm}W_{i,j} \hat{y}^{(k)}[\{i,j\}] \bigg]$    \label{algostep1} 
\vspace*{2mm}
\State for all nodes $i\!\in\!\samplingset$:  $\hat{x}^{(k+1)}[i] \defeq x_{i}$
\vspace*{2mm}
\State for all nodes $i \in \nodes$: $\tilde{x}[i]  \defeq 2 \hat{x}^{(k+1)}[i]  - \hat{x}^{(k)}[i]$    
\vspace*{2mm}
\State for all edges $e \in \overrightarrow{\edges}$: $\hat{y}^{(k+1)}[e]  \defeq \hat{y}^{(k)}[e] +   \lambda_{\{i,j\}}   (\tilde{x}[e^{+}] - \tilde{x}[e^{-}])$
\vspace*{2mm}
\State for all edges $e \in \overrightarrow{\edges}$: $\hat{y}^{(k+1)}[e]  \defeq \hat{y}^{(k+1)}[e] / \max\{1, |\hat{y}^{(k+1)}[e]| \}$  
\vspace*{2mm}
\State $k \defeq k+1$    
\vspace*{2mm}
\Until{stopping criterion is satisfied}
\vspace*{2mm}
\Ensure labels $\hat{x}_{\rm SLP}[i] \defeq \hat{x}^{(k)}[i]$ \mbox{ for all }$i \in \nodes$
\end{algorithmic}
\label{sparse_label_propagation_mp}
\end{algorithm}

\section{When is SLP Accurate?} 
\label{sec_main_results} 

In the previous section, we reduced the learning of labels $x[i]$ for all data points $z_{i}$ 
based on the data graph and knowledge of a few initial labels $x_{i}$, provided for $i \in \samplingset \subseteq \nodes$, 
to a convex optimization problem \eqref{equ_min_constr}. 
Applying an efficient primal-dual method ~\citep{pock_chambolle} to \eqref{equ_min_constr} 
resulted in the SLP Algorithm \ref{alg_sparse_label_propagation_centralized} and 
its message passing implementation Algorithm \ref{sparse_label_propagation_mp}. 

For SLP methods, based on solving \eqref{equ_min_constr}, to be accurate, i.e., to deliver labels 
$\hat{x}^{(k)}[i] \approx x[i]$ after a sufficient number $k$ of iterations, the solutions $\hat{\vx}$ to \eqref{equ_min_constr} have to be 
close to the true labels $x[i]$. 
In particular, for a true underlying clustered labeling $\vx$ of the form \eqref{equ_def_clustered_signal_model}, 
we will present a condition which guarantees any solution $\hat{\vx}$ to \eqref{equ_min_constr} 
is close to $\vx$. 
To this end, we first introduce the following definition. 
\begin{definition}
\label{def_sampling_set_resolves}
Consider a partition $\partition=  \{ \cluster_{1},\cluster_{2},\ldots,\cluster_{|\partition|} \}$ of the data graph 
$\graph$ into pairwise disjoint subsets of nodes (clusters) $\cluster_{l} \subseteq \nodes$. 
We are provided with initial labels for all data points in the sampling set $\samplingset \subseteq \nodes$. 
We say that the partition $\partition$ is resolved by a sampling set $\samplingset$ 
if for each boundary edge $\{i,j\} \in \boundary$, we can find 
two sampled nodes $m,n \in \samplingset$ such that $W_{m,i} \geq 2 W_{i,j}$ and $W_{n,j} \geq 2 W_{i,j}$. 
\end{definition} 
The sampled nodes $m,n \in \samplingset$ considered in Definition \ref{def_sampling_set_resolves} necessarily have to be connected to $i$ and $j$, i.e., $\{i,m\} \in \edges$ and 
$\{n,j\} \in \edges$, since $W_{m,i} \geq 2 W_{i,j} > 0$ (cf.\ \eqref{equ_edge_set_support_weights}) and, similarly, $W_{n,j} > 0$. 
It is important to note that Definition \ref{def_sampling_set_resolves} involves both the sampling set $\samplingset$ and the partition $\partition$. 
In Fig.\ \ref{fig_cluster}, we depict a sampling set $\samplingset \subseteq \nodes$ which resolves the partition $\partition$.  
\begin{figure}
\begin{center}
\includegraphics[width=.6\columnwidth,angle=0]{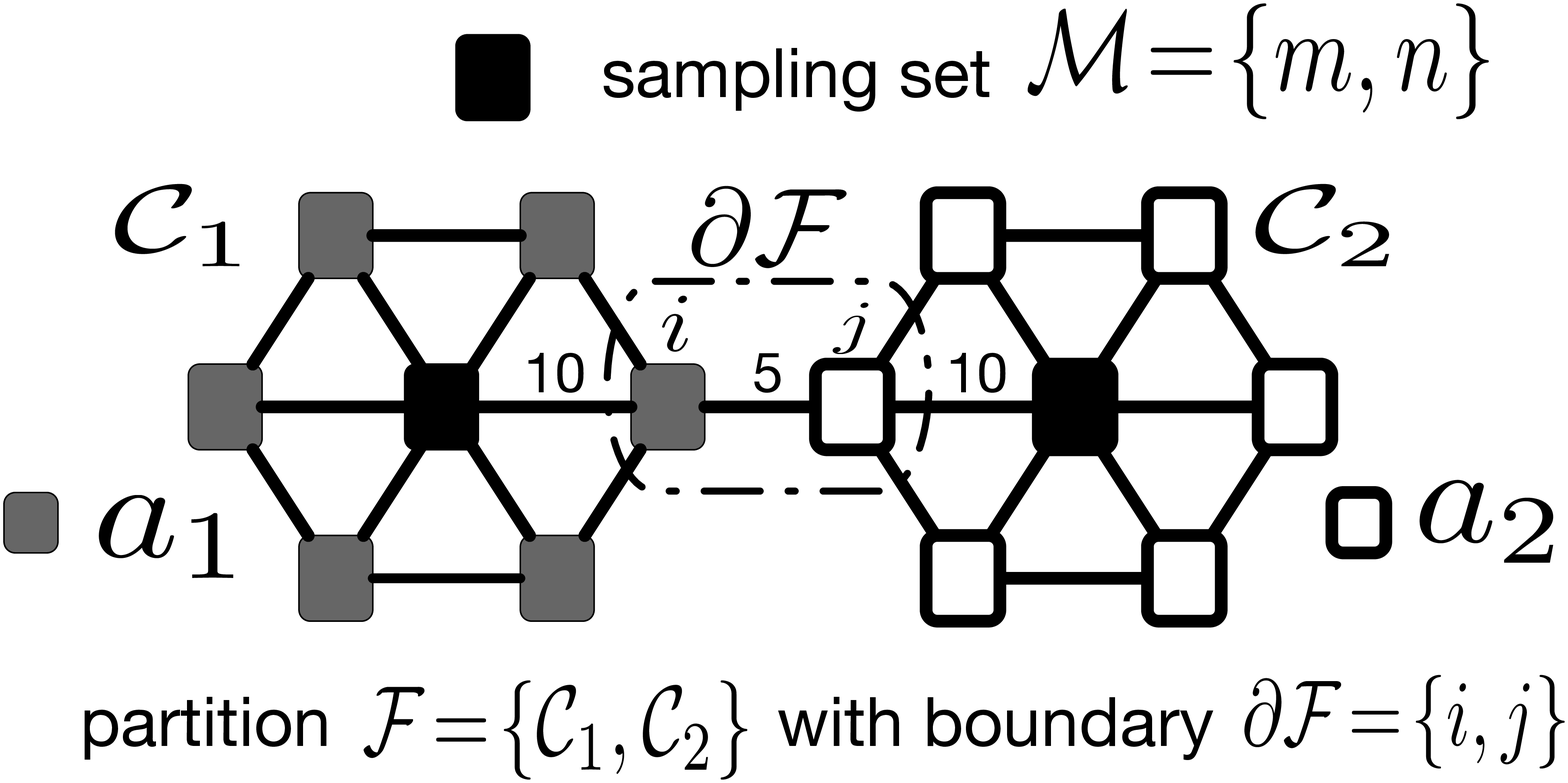}
\end{center}
\caption{\label{fig_cluster} Data graph $\graph$ composed of two clusters $\cluster_{1}$, $\cluster_{2}$. 
Initial labels $\{x_{m},x_{n}\}$ are provided for data points in the sampling set $\samplingset = \{m,n\}$.}
\end{figure}

\begin{theorem} 
\label{main_thm_exact_sparse}
Consider dataset $\dataset$ with data graph $\graph$ whose nodes represent data points with associated 
labels $\xsigval{i}$. The labels form a clustered graph signal $\vx \in \graphsigs$ 
according to \eqref{equ_def_clustered_signal_model} with a fixed partition $\partition=\{\cluster_{1},\ldots,\cluster_{|\partition|}\}$. 
Assume that we are provided the true labels $x_{i}$ on the sampling set $\samplingset$. If the sampling set 
$\samplingset$ resolves the partition $\partition$, 
the solution $\hat{\vx}$ to \eqref{equ_min_constr} is unique and coincides with the true labels, i.e., 
$\hat{x}[i] = x[i]$ for all $i \in \nodes$. 
\end{theorem} 
By Theorem \ref{main_thm_exact_sparse}, learning algorithms based on solving \eqref{equ_min_constr}, 
such as Algorithm \ref{alg_sparse_label_propagation_centralized}, can be expected to accurately learn true 
underlying labels $x[i]$ if they amount to a clustered labeling $\vx \in \graphsigs$ of the form \eqref{equ_def_clustered_signal_model}. 
We again highlight the fact that Algorithm \ref{alg_sparse_label_propagation_centralized} 
does not require knowledge of the partition $\partition=\{\mathcal{C}_{1},\ldots,\mathcal{C}_{|\partition|}\}$. 
Indeed, we could use Algorithm \ref{alg_sparse_label_propagation_centralized} for determining the 
clusters $\cluster_{l}$ if the underlying labels $x[i]$ form a clustered graph signal $x[i] = \sum_{l=1}^{|\partition|} a_{l} \mathcal{I}_{\cluster_{l}}[i]$ 
with $a_{l} \neq a_{l'}$ for different clusters $l \neq l'$.

The scope of Theorem \ref{main_thm_exact_sparse} is somewhat limited as 
it applies only to clustered graph signals of the form \eqref{equ_def_clustered_signal_model}. 
We now state a more general result which applies to any graph signal $\vx \in \graphsigs$. 
\begin{theorem} 
\label{main_thm_approx_sparse}
Consider a dataset represented by a data graph $\graph$ and true labels $x[i]$ forming a graph signal $\vx \in \graphsigs$.  
We are provided with initial labels $x_{i} = x[i]$ at nodes in the sampling set $\samplingset$. 
If the sampling set $\samplingset$ resolves $\partition=\{\cluster_{1},\ldots,\cluster_{|\partition|}$, any solution $\hat{\vx}$ 
to \eqref{equ_min_constr} satisfies
\begin{equation}
\| \hat{\vx} - \vx \|_{\rm TV} \leq 6  \min_{\{ a_{l} \}_{l=1}^{|\partition|}} \big\| \xsig- \sum_{l=1}^{|\partition|} a_{l} \mathcal{I}_{\cluster_{l}}[\cdot] \big\|_{\rm TV}. 
\end{equation} 
\end{theorem} 
Thus, as long as the true underlying labeling $\vx\!\in\!\graphsigs$ can be well approximated 
by a clustered graph signal of the form \eqref{equ_def_clustered_signal_model},  
with suitable coefficients $a_{l}$, any solution $\hat{\vx}$ to \eqref{equ_min_constr} 
is close (measured in TV semi-norm) to the true underlying labeling $\vx$.


\section{Numerical Experiments}
\label{experimental_results}

In order to asses the empirical performance of SLP Algorithm \ref{alg_sparse_label_propagation_centralized} 
and its message passing formulation Algorithm \ref{sparse_label_propagation_mp}, we carried out some 
illustrative numerical experiments involving both real-world and synthetic datasets.
For these experiments, we implemented the message passing formulation 
Algorithm \ref{sparse_label_propagation_mp} using the graph computation 
system \textsc{GraphX} ~\citep{xin2013graphx}, which, in turn, is a higher level abstraction for 
the the general purpose big data framework \textsc{Spark} ~\citep{zaharia2010spark}. 
The programming model of \textsc{GraphX} is based on a master/slave architecture which in 
our experiments consisted of up to $16$ identical worker nodes and one master node. 
Each of those master and worker nodes is equipped with a $64$ bit CPU, $8$ GB of 
RAM, $8$ GB of disk space with two \textsc{Spark} partitions. 

\subsection{Chain graph}  
\label{ssec_chainGraph}

The first experiment revolves around a synthetic dataset whose data graph $\graph_{1}=(\nodes,\edges,\mW)$ 
is a chain (cf.\ Fig.\ \ref{fig_graph_signals}-(a)). In particular, the graph $\graph_{1}$ contains a total of 
$\signalsize=10^{6}$ nodes $\nodes=\{1,2,\ldots,\signalsize\}$ which are connected by $\signalsize-1$ 
undirected edges $\edges =\{ \{i,i+1\} \}_{i=1,\ldots,\signalsize-1}$. The nodes of $\graph_{1}$ 
are partitioned into $\signalsize/5$ disjoint clusters $\partition=\{\cluster_{1},\ldots,\cluster_{\signalsize/5} \}$, 
each of which is constituted by $5$ consecutive nodes, i.e., 
\begin{equation}
\label{equ_cluster_chain}
\cluster_{1}=\{1,\ldots,5\}, \cluster_{2} = \{6,\ldots,10\},\ldots, \cluster_{\signalsize/5}=\{\signalsize-4,\ldots,\signalsize \}. 
\end{equation} 
The weights $W_{i,j}$ for the edges $\{i,j\} \in \edges$ are chosen according to the partitioning as 
\begin{equation}
\nonumber
W_{i,j} = \begin{cases} 2 & \mbox{ if } i,j \in \cluster_{l} \mbox{ for some } l \in \{1,\ldots,\signalsize/5\} \\  1 & \mbox{ else. } \end{cases}
\end{equation}

Given the data graph $\graph_{1}$, we generated labels $x[i]$ forming a clustered graph signal of the form \eqref{equ_def_clustered_signal_model} with 
clusters given by \eqref{equ_cluster_chain} and alternating coefficients $a_{l} \in \{1,5\}$, i.e., $a_{1}=1,a_{2}=5,a_{3}=1$ and so on. 
The labels $x[i]$ should then be learned solely based on the knowledge of their precise values on the sampling set $\samplingset$ 
which contains exactly one node from each cluster $\cluster_{l}$. The size of the sampling set is therefore $|\samplingset| = \signalsize/5$, i.e., 
it amounts to $20 \%$ of the entire data graph $\graph_{1}$. 

The learning accuracy achieve by $k$ iterations of SLP Algorithm \ref{sparse_label_propagation_mp} 
and ordinary LP ~\citep{Zhu02learningfrom} (based on \eqref{equ_LP_problem}) is measured by 
the normalized MSE (NMSE) $\varepsilon \defeq \| \hat{\vx}\!-\!\tilde{\vx}^{(k)} \|^{2}_{2} / \| \tilde{\vx}^{(k)} \|^{2}_{2}$. 
We run both methods for a fixed number of $200$ iterations and using the same sampling 
set $\samplingset$ and initial labels $x_{i} = \tilde{x}[i]$ for $i \in \samplingset$. 

In Figure \ref{fig_lpvslp}, we depict the true labels $\tilde{x}[i]$ as well as the 
labels delived by SLP and LP, denoted $\hat{x}_{\rm SLP}[i]$ and $\hat{x}_{\rm LP}[i]$, respectively. 
We observe that ordinary LP tends to smooth out the true labels $\tilde{x}[i]$ by pushing the learned labels 
$\hat{x}_{\rm LP}[i]$ for $i \notin \samplingset$ towards the weighted average of neighbours labels. 
In contrast, the labels $\hat{x}_{\rm SLP}[i]$ learned by SLP accurately resemble the true labels $\tilde{x}[i]$. 
The NMSE achieved (after $200$ iterations) by SLP is $\varepsilon_{\rm SLP}= 4.3 \cdot 10^{-3}$ 
while ordinary LP incurs an NMSE of $\varepsilon_{\rm LP}$ of $102.5\cdot 10^{-3}$. 
We also depict the evolution of the NMSE $\varepsilon_{\rm SLP}$ and $\varepsilon_{\rm LP}$ as a 
function of iteration number $k$ is depicted in Figure \ref{fig_mses}.  

\begin{figure}[!h]
  \centering
  \hspace*{0em}\includegraphics[width=.6\linewidth]{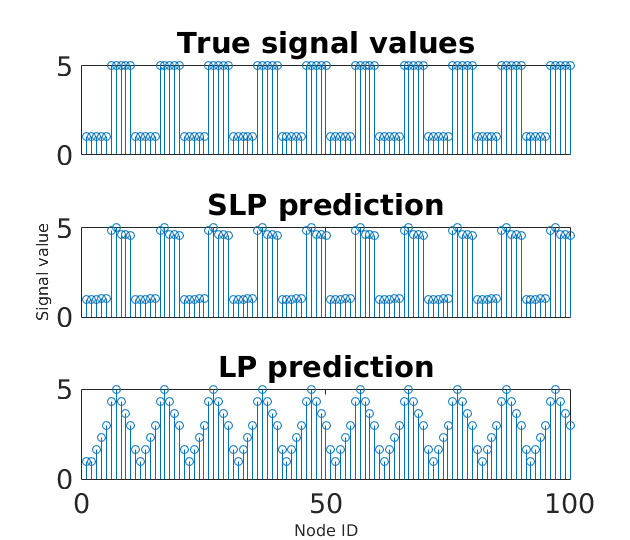}
  \caption{True labels $x[i]$ and labels learned by SLP and LP applied to the data graph $\graph_{1}$.}
  \label{fig_lpvslp}
\end{figure}

\begin{figure}[!h]
  \centering
  \hspace*{0em}\includegraphics[width=.6\linewidth]{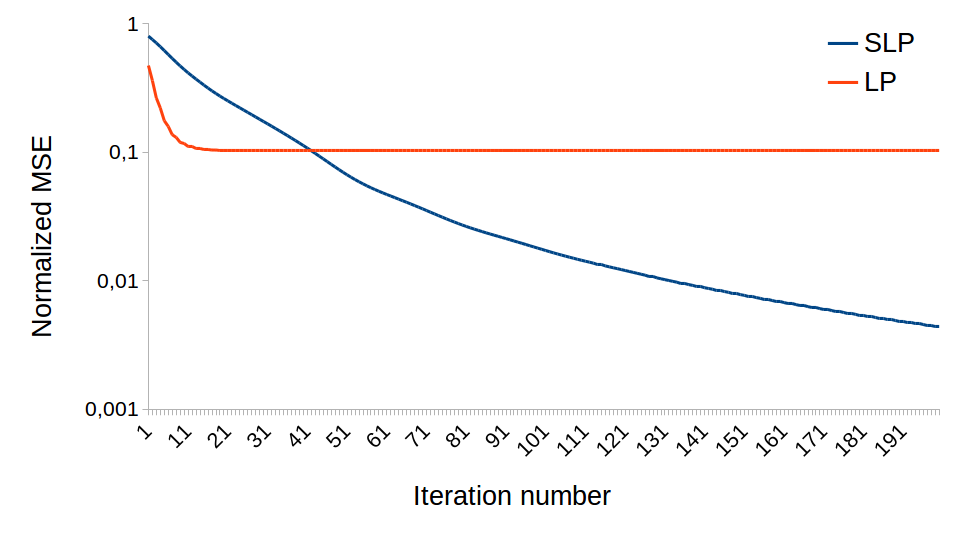}
  \caption{Dependency of NMSE achieved by SLP and LP on number of iterations, when applied to the data graph $\graph_{1}$.}
  \label{fig_mses}
\end{figure}

\subsection{LFR graph}
For our second experiment we generate a data graph $\graph_{2}=(\nodes,\edges,\mathbf{W})$ using 
the Lancichinetti-Fortunato-Radicchi (LFR) model, which is popular for benchmarking of network algorithms ~\citep{PhysRevE.78.046110}. 
The LFR model aims at mimicking the properties of real world networks 
(e.g., the internet or social networks ~\citep{NewmannBook}). In particular, 
those networks have a power law degree-distribution and exhibit a grouping of nodes 
into communities or clusters which are only weakly connected mutually. 

The data graph $\graph_{2}$ contains $\signalsize\!=\!30$ nodes $\nodes=\{1,2,\ldots,\signalsize\}$ which are partitioned into 
four clusters $\partition=\{\cluster_{1},\cluster_{2},\cluster_{3},\cluster_{4}\}$.  
The nodes are labeled according to cluster membership, 
i.e., $x[i]=1$ for all $i \in \cluster_{1}$, $x[i]=2$ for all $i \in \cluster_{2}$ and so on.  Thus, the labels induce a clustered 
graph signal of the form \eqref{equ_def_clustered_signal_model}. The graph nodes are connected by $|\edges| = 156$ undirected edges with 
weights $W_{i,j} \in [1,2]$ for all $\{i,j\} \in \edges$. 
In Figure \ref{fig_signal}, we depict the data graph $\graph_{2}$ along with the label values $x[i]$. 

\begin{figure}[!h]
\centering
\hspace*{0em}\includegraphics[width=.7\linewidth]{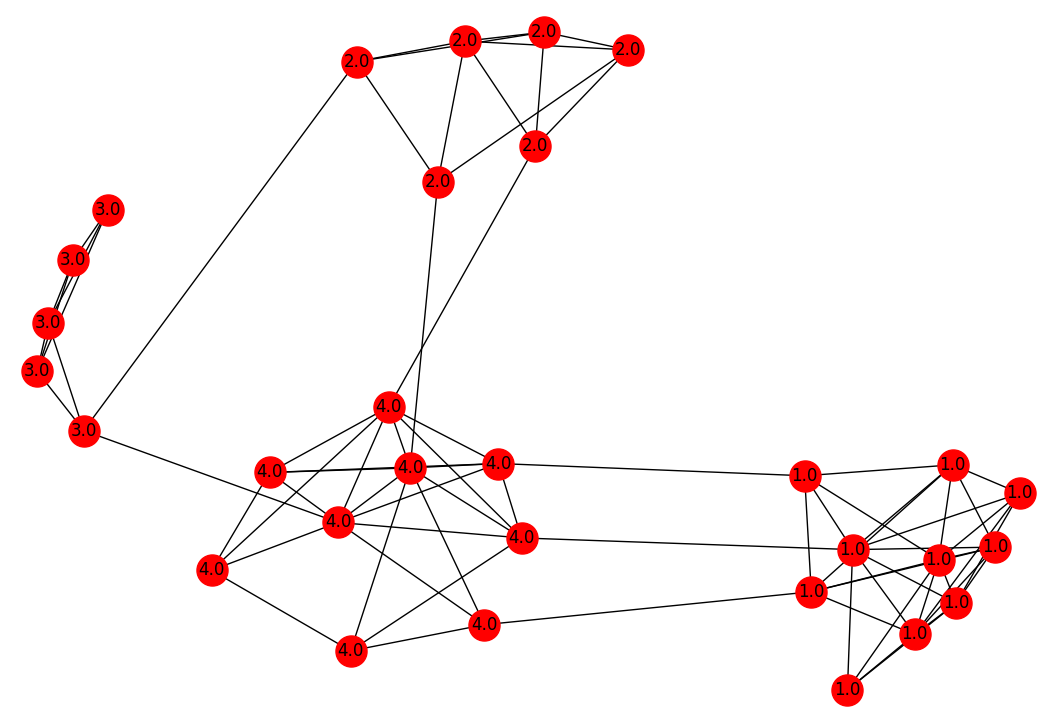}
\caption{The data graph $\graph_{2}$ conforming to the LFR model. We also indicated the labels $x[i]$ associated with nodes representing data points.}
\label{fig_signal}
\end{figure}

We then selected randomly $9$ nodes of $\graph_{2}$ as the sampling set for which the 
labels $x[i]$ are provided to SLP Algorithm \ref{sparse_label_propagation_mp} and ordinary LP, which 
both are run for a fixed number of $100$ iterations. 
The learned labels delivered by SLP and LP, along with the true labels $x[i]$ are shown In Figure \ref{fig_lpvslp_lfr}. 
Obviously the labels $\hat{x}_{\rm SLP}[i]$ delivered by SLP are significantly more close to the true labels $x[i]$, 
compared to those $\hat{x}_{\rm LP}[i]$ delivered by ordinary LP. 
The corresponding NMSE values are $\varepsilon_{\rm SLP} = 1.3 \cdot 10^{-3}$ and $\varepsilon_{\rm LP} = 12.2 \cdot 10^{-3}$, 
respectively. The convergence of $\varepsilon_{\rm SLP}$ and $\varepsilon_{\rm LP}$ is depicted in Figure \ref{fig_mses_lfr}.

\begin{figure}[!h]
  \centering
  \hspace*{0em}\includegraphics[width=.6\linewidth]{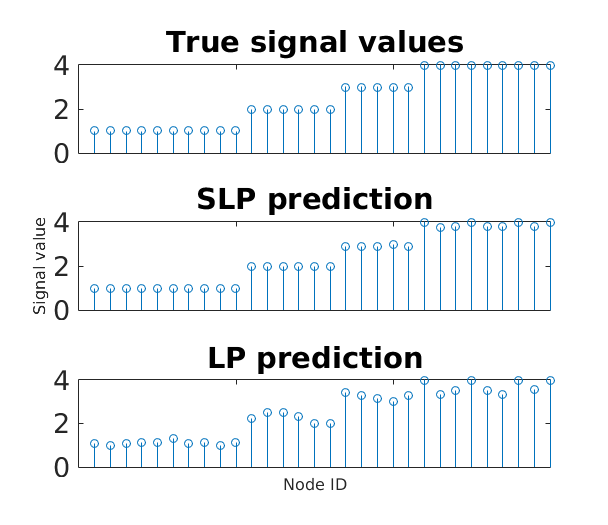}
  \caption{True labels $x[i]$, SLP output $\hat{x}_{\rm SLP}[i]$ and LP output $\hat{x}_{\rm LP}[i]$ for data graph $\graph_{2}$.}
  \label{fig_lpvslp_lfr}
\end{figure}

\begin{figure}[!h]
  \centering
  \hspace*{0em}\includegraphics[width=.6\linewidth]{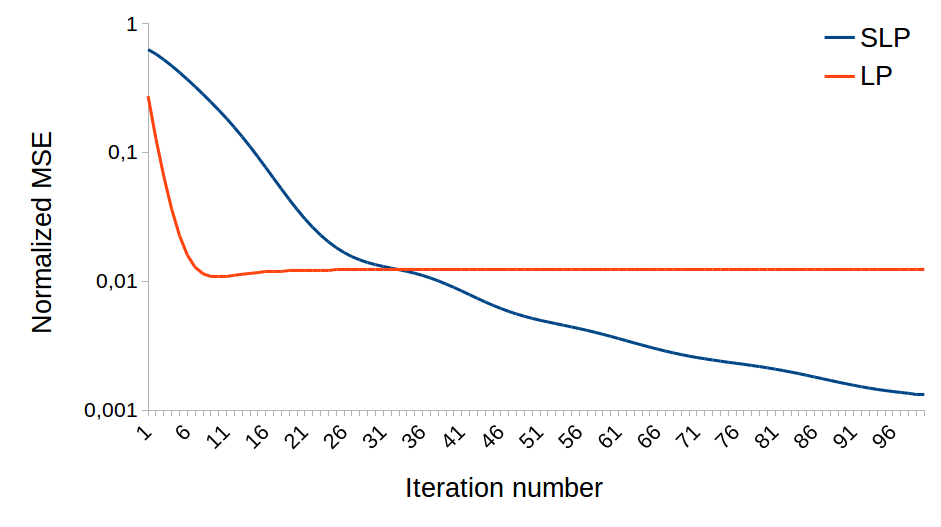}
  \caption{Convergence (of NMSE) of SLP and LP applied to data graph $\graph_{2}$.}
  \label{fig_mses_lfr}
\end{figure}

\subsection{Foreground-Background Segmentation}

Our third experiment considers the problem of segmenting 2D images into foreground and background ~\citep{Rother:2004:GIF:1015706.1015720}. 
In Figure \ref{fig_img1} we show two RGB images (taken from the "grabCut" dataset \citep{Rother:2004:GIF:1015706.1015720}) which are partitioned 
or segmented into three disjoint regions: background $\mathcal{R}_{3}$, foreground $\mathcal{R}_{1}$ 
and ``unknown'' $\mathcal{R}_{2}$. The goal is to correctly assign the pixels in $\mathcal{R}_{2}$ to foreground or background. 

\begin{figure}
\centering
\hspace*{0em}\includegraphics[width=.6\linewidth]{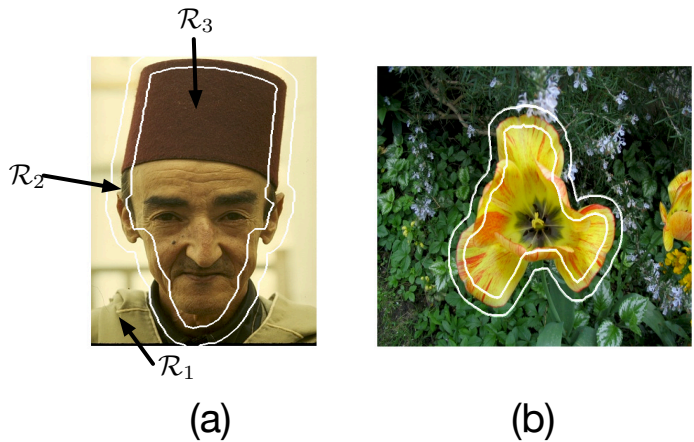}
\caption{Two images from the ``grabcut'' dataset which are segmented into three regions: 
background $\mathcal{R}_{1}$, ``unknown'' $\mathcal{R}_{2}$ and foreground $\mathcal{R}_{3}$. 
The boundaries between the regions are indicated by white stripes.}
\label{fig_img1}
\end{figure}

In order to apply SLP to this segmentation problem, we represent each image in Figure \ref{fig_img1} by a data graph $\graph_{3}$ 
having a grid structure (cf.\ Figure \ref{fig_graph_signals}). In particular, the nodes of the data graph $\graph_{3}$ represent individual 
image pixels and each node is connected to up to four nodes representing 
the left, right, top and bottom neighbouring pixel. The weight $W_{i,j}$ of an edge $\{i,j\} \in \edges$ is set according to 
$W_{i,j} \defeq \exp \big(-(1/\sigma)\| \vv[i]-\vv[j]\|_{2}^{2})$ with $\sigma \defeq {\rm median} \{ \| \vv[i]-\vv[j]\|_{2} \}_{\{i,j\}\in\edges}$. 
Here, the vector $\vv[i]=({\rm red}[i], {\rm green}[i], {\rm blue}[i])^{T} \in \{0,1,...,255\}^3$ represented the RGB code for 
the image pixel represented by node $i\in \nodes$.  

The sampling set used for SLP is $\samplingset = \mathcal{R}_{1} \cup \mathcal{R}_{3}$, where (with slight abuse of notation)
we denote by $\mathcal{R}_{1}$ and $\mathcal{R}_{3}$ the nodes in the data graph $\graph_{3}$ which represent foreground 
and background pixels, respectively. As initial labels, we use  $x[i]=1$ for the background pixels $i \in \mathcal{R}_{1}$ and 
$x[i]=-1$ for all foreground pixels $i\in \mathcal{R}_{3}$. We ran SLP Algorithm \ref{sparse_label_propagation_mp} for a fixed number 
of $500$ iterations. The signs of the learned labels $\hat{x}_{\rm SLP}[i]$ for $i \in \mathcal{R}_{2}$ are then used to determine if 
the pixel $i \in \mathcal{R}_{2}$ belongs to foreground (when $\hat{x}_{\rm SLP}[i]>0$) or background (when $\hat{x}_{\rm SLP}[i]\leq0$). 
The foreground extracted this way for the images in Figure \ref{fig_img1} are depicted in Figure \ref{fig_img_fg}. 
%
%

\begin{figure}
\centering
  \centering
  \hspace*{0em}\includegraphics[width=0.5\linewidth]{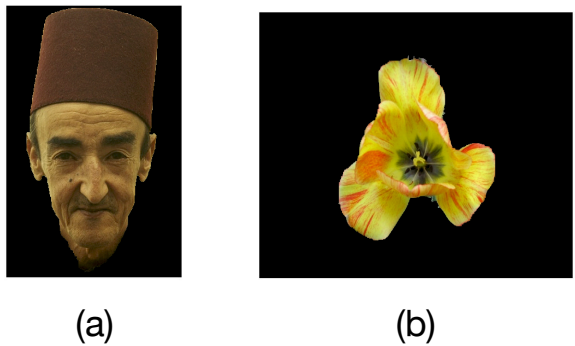}
  \caption{Foregrounds extracted from the images in Figure \ref{fig_img1} using SLP Algorithm \ref{sparse_label_propagation_mp}.}
  \label{fig_img_fg}
\end{figure}


\section{Conclusions}
\label{sec5_conclusion}
We have considered the problem of transductive semi-supervised learning 
from massive network structured datasets, i.e., big data over networks. 
The learning is based on a smoothness assumption, requiring data points within 
well-connected subsets (clusters) have similar labels. Representing the labels 
as a graph signal, we translate this informal smoothness hypothesis into a precise 
graph signal model constituted by clustered graph signals having a small 
total variation. The resulting learning problem lends then naturally to a (nonsmooth) convex 
optimization problem. By applying an efficient primal-dual method to this optimization problem, 
we obtained a sparse variant of LP which learns the labels of all data points 
from a limited amount of initial label information. 
We derived a highly scalable implementation of SLP 
in the form of message passing executed over the underlying data graph. 
Using tools from compressed sensing, we obtained a simple sufficient condition 
on the data graph structure and the set of initially labeled data points which ensures 
SLP to be accurate. This condition, roughly speaking, requires to 
have more initial labels available in the proximity of the cluster boundaries. 
We have also verified the scalability and empirical performance of SLP by means 
of numerical experiments implemented on a big data framework. 


\acks{We would like to acknowledge support for this project
from the Vienna Science Fund (WWTF) Grant ICT15-119 and US ARO grant W911NF-15-1-0479.
}

\appendix
\section*{Appendix - Proof of Main Results}

We begin with a high-level outline of the proofs for 
Theorem \ref{main_thm_exact_sparse} and Theorem \ref{main_thm_approx_sparse}:
The optimization problem \eqref{equ_min_constr} is very similar to the ``analysis  
$\ell_{1}$-minimization'' used for recovery within the cosparse analysis model of compressed sensing ~\citep{CoSparseModel}. 
A sufficient condition for analysis $\ell_{1}$-minimization to deliver the correct solution, 
i.e., the true underlying graph signal $\vx$, is the analysis nullspace property ~\citep{CoSparseModel,KabRau2015Chap}. 
Our approach is then to verify this property for graph signals of the form \eqref{equ_def_clustered_signal_model} and a sampling set $\samplingset$
which resolves the partition $\partition$ used in \eqref{equ_def_clustered_signal_model} (cf.\ Definition \ref{def_sampling_set_resolves}). 

As an intermediate step towards proving our main results Theorem \ref{main_thm_exact_sparse} and \ref{main_thm_approx_sparse}, 
we now reformulate the analysis nullspace property ~\citep{CoSparseModel,KabRau2015Chap} in graph signal terminology. 
\begin{definition} 
\label{def_NNSP}
Consider a graph $\graph$ containing 
the sampling set $\samplingset \subseteq \nodes$. 
Let us define the kernel of the sampling set $\samplingset$ as 
\begin{equation}
\label{equ_def_kernel_sampling_set}
\mathcal{K}(\samplingset) \defeq \{ \tilde{\vx} \in \graphsigs : \tilde{\vx}_{\samplingset} = \mathbf{0} \}. 
\end{equation}
The sampling set $\samplingset$ is said to satisfy the network nullspace property 
(NNSP-$\edgesupport$) w.r.t. an edge set $\edgeset \subseteq \edges$ if 
\begin{equation} 
\label{equ_NSP1}
\| (\mD \vu)_{\edges \setminus \edgeset} \|_{1} \geq 2 \| (\mD \vu)_{\edgeset} \|_{1}  \mbox{ for any } \vu \in \mathcal{K}(\samplingset)\!\setminus\! \{\mathbf{0}\}. 
\end{equation} 
\end{definition} 
Utilizing the network nullspace property, the following lemma characterizes when the learning problem \eqref{equ_min_constr} 
defines an accurate estimate for a graph signal with a small edge support $\supp (\mD \vx)$. 
\begin{lemma} 
\label{lem_NSP1}
Consider a graph signal $\vx \in \graphsigs$ with edge support 
$\edgesupport_{x}=\supp(\mathbf{D}\mathbf{x})$ which 
is observed only at the nodes in the sampling set $\samplingset \subseteq \mathcal{V}$.  
If NNSP-$\edgesupport_{x}$ holds, the solution of \eqref{equ_min_constr} 
is unique and coincides with $\vx$. 
\end{lemma}
\begin{proof}
Let us consider a graph signal $\vx \in \graphsigs$ defined over $\mathcal{G}$ and 
observed over sampling set $\samplingset$ such that condition \eqref{equ_NSP1} is satisfied. 
Assume there exists another graph signal $\hat{\vx}$ ($\neq \vx$) being feasible for 
\eqref{equ_min_constr}, i.e, $\hat{\vx}_{\samplingset} = \vx_{\samplingset}$, 
such that the difference $\vu \defeq \hat{\vx}\!-\!\vx$ belongs to the kernel 
$\mathcal{K}(\samplingset)$ (cf.\ \eqref{equ_def_kernel_sampling_set}). 
Note that, since $(\mD\vx)_{\edgesupport_{x}}=\mD\vx$, 
\begin{equation}
\label{equ_supp_vx_edgesupport}
(\mD\hat{\vx})_{\edges \setminus \edgesupport_{x}} = (\mD\vu)_{\edges \setminus \edgesupport_{x}} 
\end{equation} 
Moreover, by the triangle inequality, we have 
\begin{align}
\label{equ_proof_fist_inequ}
\| (\mD\hat{\vx})_{\edgesupport_{x}} \|_{1} & \geq \| (\mD\vx)_{\edgesupport_{x}} \|_{1} - \| (\mD\vu)_{\edgesupport_{x}} \|_{1}  \nonumber \\ 
  						&  =  \| \mD\vx \|_{1} - \| (\mD\vu)_{\edgesupport_{x}} \|_{1}.
\end{align} 
However, since $\| \mD\hat{\vx} \|_{1}  = \| (\mD\hat{\vx})_{\edgesupport_{x}} \|_{1} + \| (\mD\hat{\vx})_{\edges \setminus \edgesupport_{x}} \|_{1}$, 
\begin{align}
\| \mD\hat{\vx} \|_{1}  & = \| (\mD\hat{\vx})_{\edgesupport_{x}} \|_{1} + \| (\mD\hat{\vx})_{\edges \setminus \edgesupport_{x}} \|_{1} \nonumber \\ 
& \stackrel{\eqref{equ_supp_vx_edgesupport}}{=}\| (\mD\hat{\vx})_{\edgesupport_{x}} \|_{1} + \| (\mD\vu)_{\edges \setminus \edgesupport_{x}} \|_{1} \nonumber \\ 
 & \stackrel{\eqref{equ_proof_fist_inequ}}{\geq}  \| \mD\vx \|_{1} - \| (\mD\vu)_{\edgesupport_{x}} \|_{1} +  \| (\mD\vu)_{\edges \setminus \edgesupport_{x}} \|_{1} \nonumber \\ 
& \stackrel{\eqref{equ_NSP1}}{>} \| \mD\vx \|_{1}. 
\vspace*{-5mm}
\end{align}
\vspace*{-3mm}
\end{proof}
We will also need another more practical result which applies 
to graph signals $\vx \in \graphsigs$ whose associated 
edge signal $\mD\vx \in \edgesigs$ is not strictly sparse 
but which is well concentrated on a small subset $\edgesupport' \subseteq \mathcal{E}$ 
of edges. 
\begin{lemma} 
\label{thm_approx_sparse_stable_result}
Consider a graph signal $\vx \in \graphsigs$ which is observed 
at the nodes in the sampling set $\samplingset \subseteq \nodes$.  
If the condition NNSP-$\edgesupport'$ is valid for 
the edge set $\edgesupport' \subseteq \edges$, then any solution $\hat{\vx}$ 
of \eqref{equ_min_constr} satisfies
\begin{equation} 
\label{equ_bound_stability}
\| \mathbf{D} (\vx - \hat{\vx}) \|_{1} \leq 6 \| (\mathbf{D} \vx)_{\mathcal{E}\setminus \edgesupport'}  \|_{1}.
\end{equation} 
\end{lemma} 
\begin{proof}
The argument closely follows that in the proof of ~\citep[Theorem 8]{KabRau2015}. 
First, observe that for any solution $\hat{\vx}$ of \eqref{equ_min_constr}, 
we have 
\begin{equation}
\label{equ_D_hat_x_1_norm_smaller}
\| \mD \hat{\vx} \|_{1} \leq \| \mD \vx \|_{1}, 
\end{equation}  
since $\vx$ is trivially feasible for \eqref{equ_min_constr}. 
From \eqref{equ_D_hat_x_1_norm_smaller}, we obtain further  
\begin{equation}
\label{equ_proof_stable_NSP_123}
\| (\mD \hat{\vx})_{\edgesupport'}  \|_{1} + \| (\mD \hat{\vx})_{\mathcal{E} \setminus \edgesupport'}  \|_{1} \leq \| (\mD \vx)_{\edgesupport'}  \|_{1} + \| (\mD \vx)_{\mathcal{E} \setminus \edgesupport'}  \|_{1} . 
\end{equation}  
Since $\hat{\vx}$ is feasible for \eqref{equ_min_constr}, i.e., $\hat{\vx}_{\samplingset} = \vx_{\samplingset}$, 
the difference signal $\vv \defeq \hat{\vx} - \vx$ belongs to the kernel $\mathcal{K}(\samplingset)$ (cf.\ \eqref{equ_def_kernel_sampling_set}). 
Applying triangle inequality to \eqref{equ_proof_stable_NSP_123}, 
\begin{align} 
& \| (\mD \vx)_{\edgesupport'}  \|_{1}\!-\!\| (\mD \vv)_{\edgesupport'}  \|_{1}\!-\!\| (\mD \vx)_{\mathcal{E} \setminus \edgesupport'}  \|_{1}\!+\!\| (\mD\vv)_{\mathcal{E} \setminus \edgesupport'}  \|_{1} \leq \nonumber \\ 
&  \| (\mD \vx)_{\edgesupport'}  \|_{1} + \| (\mD \vx)_{\mathcal{E} \setminus \edgesupport'}  \|_{1}, 
\end{align} 
and, in turn, 
\begin{align} 
\| (\mD \vv)_{\mathcal{E} \setminus \edgesupport'}  \|_{1} \leq \| (\mD \vv)_{\edgesupport'}  \|_{1} + 2 \| (\mD \vx)_{\mathcal{E} \setminus \edgesupport'}  \|_{1}. 
\end{align} 
Combining this inequality with \eqref{equ_NSP1} (since the network nullspace property 
is assume to hold) gets us to 
\begin{equation}
\| (\mD \vv)_{\mathcal{E} \setminus \edgesupport'}  \|_{1} \leq (1/2) \| (\mD \vv)_{\edges \setminus \edgesupport'}  \|_{1} + 2 \| (\mD \vx)_{\mathcal{E} \setminus \edgesupport'}  \|_{1}
\end{equation} 
and, in turn, 
\begin{equation}
\label{equ_offsupport_leq_fourtimes}
\| (\mD \vv)_{\mathcal{E} \setminus \edgesupport'}  \|_{1} \leq 4 \| (\mD \vx)_{\mathcal{E} \setminus \edgesupport'}  \|_{1}. 
\end{equation} 
Using the network nullspace property again, 
\begin{align}
\| \mD (\vx - \hat{\vx}) \|_{1} & = \| \mD \vv \|_{1}  \nonumber \\ 
& = \| (\mD \vv)_{\edgesupport'}  \|_{1}\!+\!\| (\mD \vv)_{\mathcal{E} \setminus \edgesupport'}  \|_{1} \nonumber \\ 
& \stackrel{\eqref{equ_NSP1}}{\leq} (3/2) \!\| (\mD \vv)_{\mathcal{E} \setminus \edgesupport'}  \|_{1} \nonumber \\
& \stackrel{\eqref{equ_offsupport_leq_fourtimes}}{\leq} 6  \| (\mD \vx)_{\mathcal{E} \setminus \edgesupport'}  \|_{1}.
\end{align}
\end{proof} 

Let us now render Lemma \ref{lem_NSP1} and Lemma \ref{thm_approx_sparse_stable_result} for 
clustered graph signals $x[i]$ of the form \eqref{equ_def_clustered_signal_model} with a particular partition 
$\partition$ of the data graph into the clusters $\cluster_{l}$. 
In particular, we will give now a sufficient condition on the graph topology, 
characterized by the clusters $\cluster_{l} \in \partition$ used in \eqref{equ_def_clustered_signal_model}, 
and the sampling set $\samplingset$ such that the network nullspace property (cf.\ Definition \ref{def_NNSP}) is satisfied 
for the edge support $\supp(\mD \vx)$ for any clustered graph signal $x[i]$ conforming to \eqref{equ_def_clustered_signal_model}. 

\begin{lemma} 
\label{lem_NNSP_samplingset_suff_recovery}
Consider a partition $\partition=\{\cluster_{1},\ldots,\cluster_{|\mathcal{F}|}\}$ of node subsets (clusters) 
$\cluster_{l}$, which is resolved (cf. Definition \ref{def_sampling_set_resolves}) by the sampling set $\samplingset$. 
Then, 
the condition NNSP-$\edgesupport$ is satisfied for the boundary $\edgesupport = \partial \partition$. 
\end{lemma} 
\begin{proof} 
According to Definition \ref{def_sampling_set_resolves}, for any boundary edge 
$e=\{i,j\} \in \edgesupport$, we can find  two nodes $m \in \cluster_{a} \cap \samplingset$ 
and $n \in \cluster_{b} \cap \samplingset$ such that 
\begin{equation}
\label{equ_existence_sampled_nodes_boundary}
\{m,i\}, \{n,j\} \in \edges \mbox{, with }  W_{m,i} \geq 2 W_{i,j} \mbox{, and } W_{n,j} \geq 2W_{i,j}.
\end{equation} 
Let $\vu \in \mathcal{K}(\samplingset)$ be an arbitrary graph signal which vanishes 
on the sampling set $\samplingset$, i.e., 
\begin{equation}
\label{equ_sig_in_kernel_samplingset}
\vu_{\samplingset} = \mathbf{0}.
\end{equation} 
This graph signal induces an edge signal $\vv = \mD \vu$. 
We now verify  $ \| \vv_{\edges \setminus \edgesupport} \|_{1} \geq 2 \| \vv_{\edgesupport} \|_{1} $. 
Indeed, 
\begin{align}
\| \vv_{\edgesupport} \|_{1}  & = \sum_{e=\{i,j\} \in \edgesupport} W_{i,j} |u[i] -u[j]| \nonumber \\ 
    & \stackrel{u[m]=u[n]=0}{=} \sum_{e=\{i,j\} \in \edgesupport} W_{i,j} |u[i]-u[j]-u[m]+u[n]| \nonumber \\ 
    & \leq  \sum_{e=\{i,j\} \in \edgesupport} W_{i,j} |u[i]-u[m]|+ W_{i,j}|u[n]-u[j]|  \nonumber \\ 
    & \stackrel{\eqref{equ_existence_sampled_nodes_boundary}}{\leq} \sum_{e=\{i,j\} \in \edgesupport} (1/2) W_{m,i} |u[i]-u[m]|+ (1/2) W_{n,j} |u[j]-u[n]|  \nonumber \\
    & \leq (1/2) \| \vv_{\mathcal{E} \setminus \edgesupport} \|_{1} . 
\end{align}  
\end{proof} 

We can then verify Theorem \ref{main_thm_exact_sparse} by 
combining Definition \ref{def_sampling_set_resolves} with 
Lemma \ref{lem_NNSP_samplingset_suff_recovery} and Lemma \ref{lem_NSP1}. 
Similarly, we verify Theorem \ref{main_thm_approx_sparse} by combining Definition \ref{def_sampling_set_resolves} with 
Lemma \ref{lem_NNSP_samplingset_suff_recovery} and Lemma \ref{thm_approx_sparse_stable_result}.

\vskip 0.2in
\bibliography{SLPBib}

\end{document}